\documentclass[letterpage, 11pt, notitlepage]{article}
\usepackage[margin=1.0in]{geometry}

\usepackage{amsmath,graphicx}

\usepackage{microtype}
\usepackage{subfigure}
\usepackage{booktabs} 

\usepackage{algorithm}
\usepackage{algorithmic}

\usepackage[utf8]{inputenc}
\usepackage{multirow,makecell}
\usepackage{amsfonts}       
\usepackage{nicefrac}       
\usepackage{bbold}
\usepackage{capt-of}
\usepackage{etoc}
\usepackage{caption}

\usepackage{tabularx,makecell}

\usepackage{amsthm, amssymb,cite, epsfig,epstopdf,color,soul}
\allowdisplaybreaks
\usepackage{color}
\usepackage{enumerate}
\usepackage[shortlabels]{enumitem}
\usepackage{multicol}
\usepackage{empheq}
\usepackage{caption} 

\newcommand{\Acal}{{\mathcal A}}

\newcommand{\Fcal}{{\mathcal F}}

\newcommand{\Lcal}{{\mathcal L}}

\newcommand{\Ocal}{{\mathcal O}}

\newcommand{\Rcal}{{\mathcal R}}

\newcommand{\1}{{\mathbf{1}}}

\newcommand{\argmin}{\mathop{\rm argmin}}
\newcommand{\argmax}{\mathop{\rm argmax}}

\newtheorem{thm}{Theorem}

\newcommand{\var}{\operatorname{var}}

\newcommand{\vct}[1]{\boldsymbol{#1}}

\newcommand{\vh}{\vct{h}}

\newcommand{\vx}{\vct{x}}

\newtheorem{theorem}{Theorem}[section]

\newtheorem{lemma}[theorem]{Lemma}

\newtheorem{proposition}{Proposition}
\usepackage{comment}
\usepackage{hyperref}
\usepackage{url}

\title{Partially Observable Contextual Bandits with Linear Payoffs}

\author{%
  Sihan Zeng\thanks{The authors contributed equally.} \quad Sujay Bhatt\footnotemark[1] \quad Alec Koppel \quad Sumitra Ganesh\\
  JPMorgan AI Research, United States.
}

\begin{document}
%
\maketitle
\begin{abstract}

The standard contextual bandit framework assumes fully observable and actionable contexts.
In this work, we consider a new bandit setting with partially observable, correlated contexts and linear payoffs, motivated by the applications in finance where decision making is based on market information that typically displays temporal correlation and is not fully observed.
We make the following contributions marrying ideas from statistical signal processing with bandits:
(i) We propose an algorithmic pipeline named \textbf{\texttt{EMKF-Bandit}}, which integrates system identification, filtering, and classic contextual bandit algorithms into an iterative method alternating between latent parameter estimation and decision making.
(ii) We analyze \texttt{EMKF-Bandit} when we select Thompson sampling as the bandit algorithm and show that it incurs a sub-linear regret under minimal conditions on filtering.
(iii) We conduct numerical simulations 
that demonstrate the benefits and practical applicability of the proposed pipeline. 
\end{abstract}



\section{Introduction}
We study contextual bandits where the context is not fully observable, a setting that significantly departs from the classic literature.
Consider the problem of making trading decisions where the arms correspond to different algorithmic trading strategies and the reward is the monetary gain.
The reward is a function of the evolving market condition (context) with potential influences from various exogenous factors like Twitter feeds, secondary market behaviour, local trends, etc, of which only a small subset can be directly observed. Large institutional investors may spend additional resources on tracking other relevant features that reveal information on the true underlying context. The goal is to quickly identify and play the best strategy that maximizes the cumulative gain from trading.

Motivated by problems of this nature, we introduce and study a partially observable linear contextual bandit framework, where a decision maker interacts with an environment over $T$ rounds. 
The (latent) context evolves according to a linear dynamical system -- with context at time $t=1,2,\cdots,T$ denoted by $x_t\in\mathbb{R}^d$, we have
\begin{align}
x_{t+1}=D x_t+\epsilon_t,\quad \epsilon_t \sim N(0,Q), \label{eq:transition}
\end{align}
where the transition matrix $D\in \mathbb{R}^{d\times d}$ is unknown to the decision maker and the noise covariance matrix $Q\in \mathbb{R}^{d\times d}$ is unknown and positive definite.

The decision maker does not directly observe the contexts, but indirectly through linear observations. In this work we assume the observation model is linear. Before an action is taken in time $t$, the decision maker observes $y_t\in\mathbb{R}^{k}$ such that
\begin{align}
y_t=Ax_t+n_t,\quad n_t\sim N(0,\Sigma)\label{eq:partial_observation}
\end{align}
for a full-rank (under-determined) observation matrix $A\in\mathbb{R}^{k\times d}$ ($k\ll d$) and noise covariance matrix $\Sigma\in\mathbb{R}^{k\times k}$. 
There is a finite number of $K$ actions, and we denote the action space by $\Acal=\{1,\cdots,K\}$. We consider the linear reward setting, where each action $a\in\Acal$ is associated with a parameter $\mu_a\in\mathbb{R}^{d}$. If action $a_t$ is selected in time $t$, we observe reward
\begin{align}
r_t=\langle x_t,\mu_a\rangle+\omega_t,\quad \omega_t\sim N(0,\sigma^2),\label{eq:linear_reward}
\end{align}
for some noise covariance $\sigma>0$.
We define the filtration $\Fcal_t\triangleq\{y_1,\cdots,y_t,r_1,\cdots,r_{t-1}\}$ which contains all observed randomness information before an action is taken in time $t$.

We assume in the paper that $A$ is a truncated identity matrix denoted as $I_{k\times d}$, i.e. $I_{d,d}$ with rows from $k+1$ to $d$ removed, since if it is not, we can always find an equivalent system
\begin{align*}
\tilde{x}_{t+1}=\tilde{D}\tilde{x}_t+\tilde{\epsilon}_t, \; y_t=I_{k\times d}\tilde{x}_t,\; r_t=\langle\tilde{x}_t, \tilde{A}^{-\top}\mu_a\rangle+\omega_t,
\end{align*}
with $\tilde{D}=\tilde{A}D\tilde{A}^{-1}$ and $\tilde{\epsilon}_t=\tilde{A}\epsilon_t$,
which produces the same observations and rewards. Here $\tilde{A}\in\mathbb{R}^{d\times d}$ can be any full-rank matrix such that $\tilde{A}_{i,j}=A_{i,j}$ for all $i\leq k$.

The problem reduces to the standard linear contextual bandit with $k=d$
and $\Sigma\rightarrow0$, for which Thompson sampling and upper confidence bound (UCB) algorithms are known to be optimal \cite{agrawal2013thompson, chu2011contextual}. However, if applied directly to the partial observability setting with the observation treated as the true context, these algorithms would generally fail to select the correct actions and incur a linear regret.
For the classic bandit algorithms to be effective, it is important that we recover the latent contexts from the observations, though a perfect recovery is in general infeasible due to inherent information loss. 

\noindent \textbf{Main Contributions}: In this work we propose an algorithmic pipeline, \textbf{\texttt{EMKF-Bandit}}, 
which runs online with a constant computational cost in each iteration. The first step is latent context estimation, in which given historical observations $\Fcal_t$ our goal is to obtain $\hat{x}_t$, a reliable estimate of $x_t$. If we had the knowledge of the transition and observation model, it is well-known that a Kalman filter would generate the optimal latent context estimates in the sense of mean squared error \cite{anderson2005optimal}. As we do not know the model, we take an expectation-maximization (EM) approach to jointly estimate the model parameters along with the latent context. The expectation step determines the latent context given estimated model parameters with a Kalman filter, while the maximization step calculates the maximum likelihood model parameters in closed form given the estimated contexts.
The second component of \texttt{EMKF-Bandit} is an off-the-shelf contextual bandit algorithm that is optimal in the fully observable context setting, which makes arm selections and learns $\mu_a$ pretending that $\hat{x}_t$ is the true context.
We characterize the worst-case regret of \texttt{EMKF-Bandit} under an assumption on the estimation error. We also show that \texttt{EMKF-Bandit} outperforms alternative methods in numerical simulations.\\

\noindent \textbf{Most Relevant Works}: The papers that most closely relate to our work include \cite{park2022efficient,nelson2022linearizing,park2024thompson}.
The authors in \cite{park2022efficient} study a setting similar to \eqref{eq:transition}-\eqref{eq:linear_reward}, whereby a reinforcement-learning-type algorithm is applied. Taking an approach different from our work, \cite{park2022efficient} does not explicitly attempt to recover the latent context and has no regret guarantees.
The formulation in \cite{nelson2022linearizing} is that the latent context is a discrete random variable that transitions as a time-invariant Markov chain and the observation is a i.i.d. continuous random variable conditioned on the context. The authors in
\cite{nelson2022linearizing} show that as the latent context is discrete and finite, the partially observable contextual problem reduces to a linear bandit. This reduction cannot be carried out when the latent context is continuous as in our setting. In \cite{park2024thompson}, the authors consider a partially observable bandit setting which is slightly more general than ours \eqref{eq:transition}-\eqref{eq:linear_reward} and propose a Thompson Sampling with Partial Contextual Observation (thereafter referred to as TS-PCO) algorithm, which aims to learn parameters $\{b_a\}_{a\in\Acal}$ that minimizes $\var(x_t^{\top}\mu_a-y_t^{\top}b_a)$ subject to $\mathbb{E}[x_t^{\top}\mu_a-y_t^{\top}b_a]=0$. In essence, TS-PCO \textit{does not} attempt to recover the unobserved latent context but pretends that the reward is a linear function of the `observable context' only. As we will illustrate through simulations in Section~\ref{sec:simulations}, TS-PCO in general suffers a linear regret against a policy that observes the full context due to the gap between $x_t^{\top}\mu_a$ and $y_t^{\top}b_a$. Our work investigates enhancing the regret through latent context estimation.
Finally, we note that our work also connects to \cite{park2021analysis,guo2024online} which consider the setting of linear dynamical model \eqref{eq:transition}, linear reward \eqref{eq:linear_reward}, and noise-corrupted observation
\[y_t=Ax_t+n_t,\]
where $A$ is a square matrix (i.e. $k=d$).
Our setting is significantly more challenging due to the information loss resulting from $A$ being under-determined.\\

\noindent \textbf{Other Related Literature}: The contextual bandit problem in the stationary, fully observable setting has been well studied. The pioneering works \cite{bubeck2012regret,lu2010contextual} consider the problem with finitely many contexts. In the case of continuous context, existing works mostly build on the hypothesis that the reward is a linear statistical model between the context and unknown parameter and consider both the adversarial setting (context can be generated by any arbitrary adversary) \cite{chu2011contextual,agrawal2013thompson} and stochastic setting \cite{abbasi2011improved} (context is generated i.i.d. from some distribution).
The stochastic setting has also been studied under a generalized linear reward model \cite{filippi2010parametric}. \looseness=-1

On contextual bandits under confounding/latent information, 
\cite{maillard2014latent,hong2020latent} study a generalization of the information structure typically considered in bandits, where additional causal relationships on hidden states are driving the reward process.
In \cite{lattimore2016causal}, it is demonstrated that this setting generalizes linear stochastic contextual bandits, and a best arm identification strategy is developed by estimating the transition probabilities of the causal graph to drive arm selection with maximum sample average reward. Similarly, \cite{zhou2016latent} contributes a bandit model that executes belief propagation over the latent states.

\section{EMKF-Bandit Framework}
To build up the algorithmic pipeline and introduce the important techniques step by step, we start by considering a simplified problem where the system parameters $D$ and $Q$ are known.\looseness=-1

\subsection{Known Transition Model}\label{sec:known_model}
Context estimation under a known system transition and observation model can be achieved by the classic Kalman filter, which is an online algorithm with constant amount of computation in each iteration.
Let $x_{t\mid t'}\in\mathbb{R}^d$, $P_{t\mid t'}\in \mathbb{R}^{d\times d}$ denote the estimate of the latent context and its covariance in time $t$ given observation up to time $t'$.
The Kalman filter starts with some initial guess $x_{0\mid 0}$ and $P_{0\mid 0}$ and proceeds in the following manner at any time $t$
\begin{align}
    \text{Prediction:}\;\hspace{-6pt}&\hspace{6pt} x_{t\mid t-1}\hspace{-2pt}=\hspace{-2pt}Dx_{t-1\mid t-1},\, P_{t\mid t-1}\hspace{-2pt}=\hspace{-2pt}DP_{t-1\mid t-1}D^{\top}\hspace{-3pt}+\hspace{-1pt}Q.\notag\\
    \text{Update:}\;& e_t=y_t-Ax_{t\mid t-1},\, S_t=AP_{t\mid t-1}S^{\top}+\Sigma,\notag\\ &K_t=P_{t\mid t-1}A^{\top}S_t^{-1},\notag\\
    &x_{t\mid t}=x_{t\mid t-1}\hspace{-2pt}+\hspace{-2pt}K_t e_t,\,P_{t\mid t}=(I\hspace{-2pt}-\hspace{-2pt}K_t A)P_{t\mid t-1}.\hspace{-5pt}\label{eq:Kalman_filter:eq2}
\end{align}


\noindent The Kalman filter provides the optimal estimate of the latent context in the sense that it minimizes the mean squared reconstruction error $\|x_{t|t}-x_t\|^2$ in expectation \cite{anderson2005optimal}, given observations made so far. In light of later observations, however, the estimation of the past latent context can be further refined and optimized with the Rauch–Tung–Striebel smoother \cite{rauch1965maximum}.

\subsection{Unknown Transition Model}

When the system model is unknown, it is natural to consider estimating the model and the context jointly through an EM algorithm. Specifically, we denote $\vh_t=\{x_{t'}\}_{t'\leq t}$ and $\theta=(D,Q)$. With observations up to time $t$, the EM algorithm aims to find the maximum likelihood estimate of $(\vh_t, \theta)$
\[
\Lcal_t(\vh_t,\theta)\triangleq\mathbb{P}(\widetilde{\vx}_t,\theta\mid y_1,\cdots,y_t)
\]
by maintaining estimates $\hat{\vh}_t,\hat{\theta}$ and repeatedly updating them in an alternating fashion
\begin{align*}
\text{Expectation (E-step):}\quad&\textstyle\hat{\vh}_t=\argmax_{\vh_t} \Lcal_t(\vh_t,\hat{D})\\
\text{Maximization (M-step):}\quad&\textstyle\hat{D}=\argmax_D \Lcal_t(\hat{\vh}_{t+1},D)
\end{align*}
We note that \cite{liu2022expectation} takes exactly such an approach in a problem setting different from ours, where they solve the expectation step optimally with the Rauch–Tung–Striebel smoother and the maximization step with a numerical solver. A computational constraint, however, prevents us from using the Rauch–Tung–Striebel smoother in our setting. 
Unlike \cite{liu2022expectation} which only needs to run the EM algorithm once when all observations are collected, we have to generate a latent context estimate in every iteration to feed into the bandit algorithm.
To compute $\{x_{t'\mid t}\}_{t'\leq t}$ (the estimate of latent state $x_{t'}$ with samples up to time $t\geq t'$) in time $t$, the Rauch–Tung–Striebel smoother needs to pay $\Theta(t)$ amount of computation in each iteration $t$, which quickly becomes intensive as the trajectory length increases. 

We perform the expectation step with the (optimal online, sub-optimal in hindsight) Kalman filter without smoothing. 
We consider running exactly one expectation and maximization step for each collected observation. The aim of the maximization step is to obtain the MLE of system transition parameters $D$ and $Q$. Given historical observations $y_1,\cdots,y_t$ and estimated contexts $\hat{x}_1,\cdots,\hat{x}_t$, the MLE of $D$, which we denote by $\hat{D}_t$ given observation up to time $t$, coincides with the least squares solution
\begin{align}
\textstyle\hat{D}_t=\argmin_D\sum_{t'\leq t} \|D\hat{x}_{t'-1}-\hat{x}_{t'}\|^2.
\end{align}
A recursive update rule is available which does not require storing all historical information. 
With initialization $\Psi_0=0_{d\times d}$ and $X_0=0_{d\times d}$, we have
\begin{gather*}
\Psi_t\hspace{-2pt}=\hspace{-2pt}\Psi_{t-1}\hspace{-2pt}+\hspace{-2pt}x_t x_{t-1}^{\top},\;\, X_t\hspace{-2pt}=\hspace{-2pt}X_{t-1}\hspace{-2pt}+\hspace{-2pt}x_{t-1} x_{t-1}^{\top},\;\,\hat{D}_t\hspace{-2pt}=\hspace{-2pt}\Psi_t X_t^{-1}.
\end{gather*}

Similarly, the MLE of Q, which we track by $\hat{Q}_t$, can be shown to follow the following update rule with the initialization $Y_0=0_{d\times d}$  \cite{kriouar2023learning}[Section 2.5]
\begin{gather}
Y_t=Y_{t-1}+x_{t} x_{t}^{\top},\quad
\hat{Q}_t=\frac{1}{t}(Y_t-\hat{D}_t(\Psi_t)^{\top}).
\end{gather}

A drawback of using Kalman filter without smoothing in the expectation step is that the imperfectly calculated latent contexts introduce inaccuracy to the system parameter identification in the maximization step, which may in turn leads to further degradation in latent context estimation. To break the potential vicious cycle, we periodically restart the system parameter estimation step. Given a window length $L$, we update $\hat{D}_t$ and $\hat{Q}_t$ only at iterations $T_{\text{M-step}}=\{nL\}_{n=1,2,3,\cdots}$, with observations and estimated contexts in the latest window, and hold them fixed at $t\notin T_{\text{M-step}}$.

We treat the estimated latent context as the true context, and feed it online to a standard linear contextual bandit algorithm, such as Thompson sampling \cite{agrawal2013thompson} or LinUCB \cite{chu2011contextual}. In Algorithm~\ref{alg:main}, we present the pseudo-code for the pipeline when the bandit algorithm is linear contextual Thompson sampling.

\begin{algorithm}
\caption{\texttt{EMKF-Bandit} with Thompson Sampling}
\label{alg:main}
\begin{algorithmic}[1]
\STATE{\textbf{Initialization:}  $\hat{B}_a(1)= I_{d\times d}$, $\hat{f}_a(1)=0_{d}$, model estimates $\hat{D}_0$, $\hat{Q}_0$, window length $L$, parameter $v_t$}
\FOR{$t=1,\cdots,T$}
\STATE{Observe $y_t$. Update Kalman filter parameters and obtain $\hat{x}_t$ using $y_t,\hat{D}_t,\hat{Q}_t$}
\IF{$t$ is a multiple of $L$}
    \STATE{Estimate $\hat{D}_t$, $\hat{Q}_t$ using $\{\hat{x}_{t-L+1},\hat{x}_{t-L+2},\cdots,\hat{x}_{t}\}$}
\ENDIF
\FOR{$a=1,\cdots,K$}
    \STATE{Sample $\widetilde{\mu}_{a}(t)\sim N(\hat{\mu}_a(t),v_t^2 \hat{B}_a^{-1}(t))$}
\ENDFOR
\STATE{Select $a_t=\argmax_a  \hat{x}_t^{\top} \widetilde{\mu}_{a}(t)$ and observe reward $r_t$}
\STATE{Update Thompson sampling algorithm parameters 
\begin{gather}
\hat{B}_{a}(t+1)=\hat{B}_{a}(t)+\hat{x}_t \hat{x}_t^{\top}\1(a=a_t),\notag\\
\hat{f}_{a}(t+1)=\hat{f}_{a}(t)+\hat{x}_t r_t\1(a=a_t), \notag\\
\hat{\mu}_{a}(t+1)=\hat{B}_{a}^{-1}(t+1) \hat{f}_{a}(t+1).\label{eq:muhat_update}
\end{gather}
}
\ENDFOR
\end{algorithmic}
\end{algorithm}


\section{Regret Analysis}


In this section, we analyze the regret of \texttt{EMKF-Bandit} with linear contextual Thompson sampling as the bandit algorithm (i.e. Algorithm~\ref{alg:main}). We measure the regret with respect to the optimal arm conditioned on the true context. Specifically, we denote by $a^{\star}(t)$ the action that maximizes the expected reward under context $x_t$
\begin{align}
\textstyle a^{\star}(t) = \argmax_{a}\langle x_t, \mu_a\rangle.
\end{align}

\noindent The instantaneous regret of taking action $a$ in time $t$ is
\[\Delta_a(t)=\langle x_t,\mu_{a^{\star}(t)}\rangle-\langle x_t,\mu_{a}\rangle.\]
The cumulative regret of an algorithm over $T$ intervals is
\[\textstyle\Rcal(T)=\sum_{t=1}^{T}\Delta_t(a_t),\]
where actions $\{a_t\}$ are selected by the algorithm.

\noindent Our analysis relies on the assumptions below.

\begin{itemize}
    \item[\textbf{A1}] $\|\mu_a\| \leq M_1,\,\forall a \in \Acal$, for some~$M_1< \infty$.
    \item[\textbf{A2}] $\|x_t \|, \|\hat{x}_t \| \leq M_2,\,\forall t=1,\cdots,T$, for some $M_2<\infty$. 
    \item[\textbf{A3}] 
    There exists a constant $b^{\delta}<\infty$ such that with probability at least $1 - \frac{\delta}{2t^2}$
    \[
    \textstyle\|\hat{B}_a^{-1}(t)\|\leq b^{\delta}/t,\;\forall a\in\Acal,\,t=1,\cdots,T.
    \] 
\end{itemize}

\textbf{A1} is a standard assumption on the boundedness of the true reward parameter.
\textbf{A2} requires that the states and state estimates evolve in a compact subset of~$\mathbb{R}^d$. 
\textbf{A3} amounts to an ``incoherence'' condition on the predicted contexts $\hat{x}_t$ (across time) and can be shown to hold if 
each $\hat{x}_{t+1}$ has a non-zero energy along the direction of every orthobasis of $\hat{B}_{a_t}(t)$. Precisely, with $V_{t}\Lambda_{t} V_{t}^{-1}$ representing the eigendecomposition of $\hat{B}_{a_t}(t)$ and $v_{i,t}$ denoting the $i_{\text{th}}$ row of $V_{t}$, \textbf{A3} holds if there exists a constant $\ell>0$ such that we have for all $t$\looseness=-1
\begin{align}
v_{i,t}^{\top}\hat{x}_{t+1}\geq \ell,\;\forall i=1,\cdots,d.\label{eq:A3_sufficientcondition}
\end{align}
If $\hat{x}_t$ is generated from a Kalman filter (as in our case), \textbf{A3} does not have to be assumed and we can show that \eqref{eq:A3_sufficientcondition} is true due to the presence of extrinsic Gaussian noise (from $y_t$) in the estimate of of $x_{t\mid t}$ (see Eq.~\eqref{eq:Kalman_filter:eq2}).
However, to make our main theorem statement generally hold for any $\hat{x}_t$ (constructed using any arbitrary method), we need \textbf{A3}, which importantly guarantees that the predicted context covariance matrix $\hat{B}_a(t)$ reasonably ``covers'' the true context covariance.
\looseness=-1

\begin{thm} \label{thm:main}
    Let $\varepsilon_t \triangleq x_t - \hat{x}_t$. Under the assumptions \textbf{A1-A3} and the parameter choice $v_t=\sigma\sqrt{9d\ln(\frac{t}{\delta})}$, the regret of Algorithm~\ref{alg:main} satisfies with probability~$1-\delta$ 
    \begin{align*}
    \textstyle \mathcal{R}(T) &\leq O\Big(d^{3/2}\sqrt{T \ln(K)} \big(\ln(T) + \sqrt{\ln(T) \ln(1/\delta)}\;\big) \notag\\
    &\hspace{80pt}+ \textstyle\sqrt{d \ln(T)}  \sum_{t=1}^{T} \|\varepsilon_{t}\| \Big).
    \end{align*}
\end{thm}
\noindent When the prediction errors accumulate sub-linearly, i.e.
\[
    \textstyle\sum_{t = 1}^{T} \| \varepsilon_{t}\| \leq o(T),
\]
the cumulative regret of the proposed algorithm also grows sub-linearly over time. We will show  in Section~\ref{sec:simulations} that such a sub-linear regret is indeed observed in simulations and makes the proposed algorithm advantageous over the best-known alternative method TS-PCO which runs Thompson sampling with the partial observations treated as the true contexts. 

Additionally, if the prediction errors satisfy
\[
\textstyle\sum_{t = 1}^{T} \| \varepsilon_{t}\| \leq \Ocal(\sqrt{T}),
\]
the regret of the proposed algorithm is on the same order (in $T$) as that of the standard linear contextual Thompson sampling algorithm with the full context observation \cite{agrawal2013thompson}.

\begin{figure}[!ht]
  \centering
  \vspace{-4pt}
  \includegraphics[width=.8\linewidth]{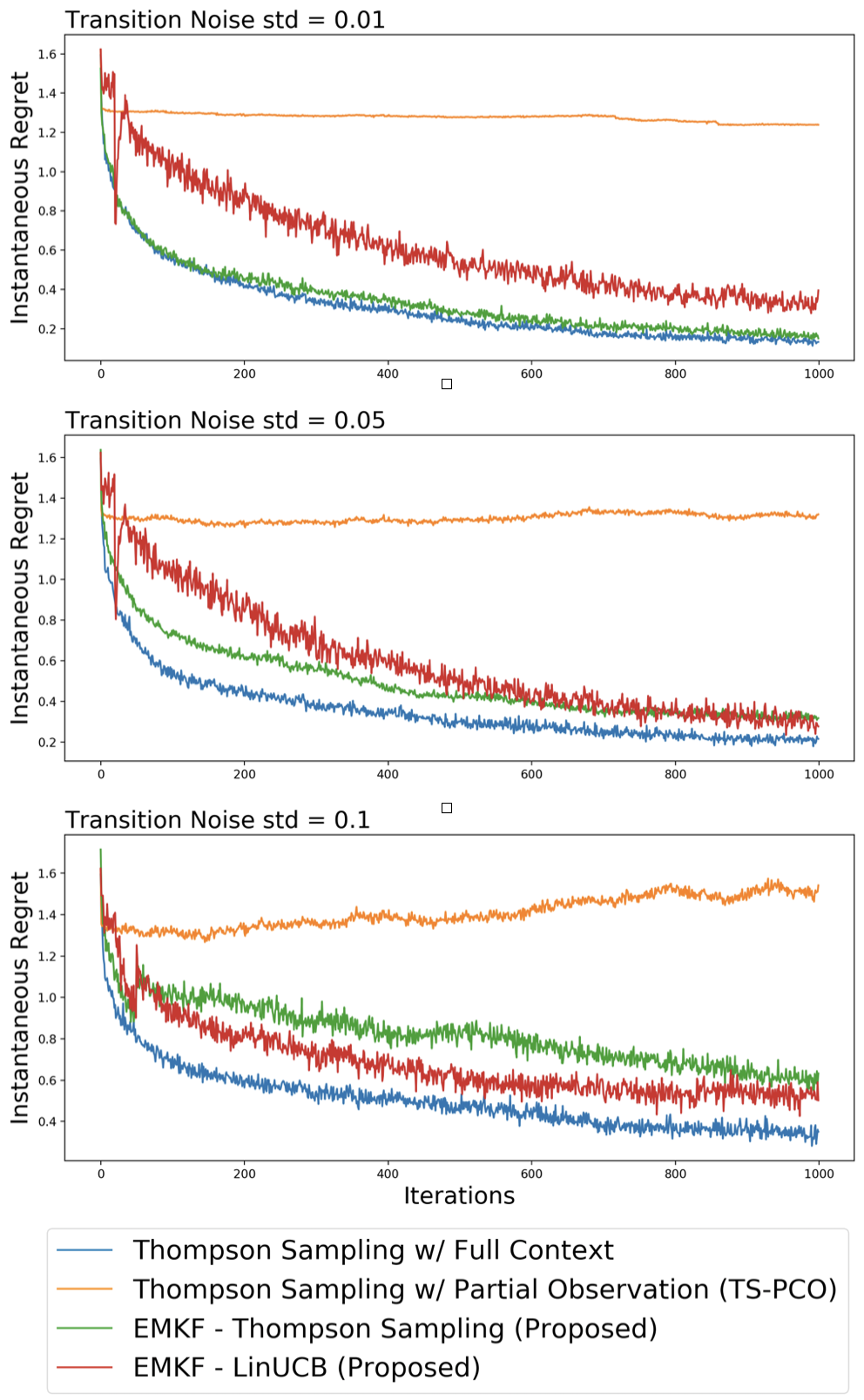}
  \vspace{-6pt}
  \caption{Algorithm Performance Under Various Noise Levels}
  \label{fig:results}
\end{figure}

\section{Experimental Results}\label{sec:simulations}

\vspace{-4pt}
This section presents a set of experiments that show the effectiveness of the proposed algorithm.
We choose the latent dimension $d=20$, observation dimension $k=10$, and number of arms $K=15$. We sample the transition $D$ entry-wise i.i.d. from a standard uniform distribution, normalized to be row stochastic. The transition noise is $Q=\varepsilon I_{d\times d}$ with $\varepsilon$ varying across experiments. The reward variance is $\sigma^2=1$. The initial context $x_0$ and the bandit system parameters $\{\mu_a\}$ are generated entry-wise i.i.d. from a standard Gaussian distribution.\looseness=-1

We compare the performance of \texttt{EMKF-Bandit} with Thompson sampling and LinUCB as the bandit backbones against (1) Thompson sampling operating on the true context, and 2) TS-PCO proposed in \cite{park2024thompson} which run Thompson sampling with partial observations treated as full contexts. 
With the instantaneous regret shown in Figure.~\ref{fig:results}, the proposed algorithms incurs a sub-linear regret that closely matches that of baseline (1), which relies on the unobservable latent contexts and is not practically implementable. In addition, the superiority of the proposed algorithms over baseline (2) are consistently observed over the range of noise levels. Notably, without latent context reconstruction, baseline (2) suffers a linear regret.

\section*{Disclaimer}
This paper was prepared for informational purposes [“in part” if the work is collaborative with external partners]  by the Artificial Intelligence Research group of JPMorgan Chase \& Co. and its affiliates ("JP Morgan'') and is not a product of the Research Department of JP Morgan. JP Morgan makes no representation and warranty whatsoever and disclaims all liability, for the completeness, accuracy or reliability of the information contained herein. This document is not intended as investment research or investment advice, or a recommendation, offer or solicitation for the purchase or sale of any security, financial instrument, financial product or service, or to be used in any way for evaluating the merits of participating in any transaction, and shall not constitute a solicitation under any jurisdiction or to any person, if such solicitation under such jurisdiction or to such person would be unlawful.

\bibliographystyle{IEEEbib}
\bibliography{references}

\newpage
\appendix
\onecolumn
\vbox{%
\hsize\textwidth
\linewidth\hsize
\hrule height 4pt
\vskip 0.25in
\centering
{\Large\bf{Partially Observable Contextual Bandit with Linear Payoffs\\ 
Supplementary Materials} \par}
  \vskip 0.25in
\hrule height 1pt
\vskip 0.09in
}

\section{Preliminaries}
\begin{lemma} [\cite{agrawal2013thompson},~Lemma~8] \label{lem:mnorm_ub}
Let $(\mathcal{F}_t;~t\geq 0)$ be a filtration, $(m_t; t\geq 1)$ be an $\mathbb{R}^d-$valued stochastic process such that~$m_t$ is $(\mathcal{F}_{t-1})-$measurable,~$(\eta_t; t\geq 1)$ be a real-valued martingale difference process such that~$\eta_t$ is $(\mathcal{F}_t)-$measurable. For~$t\geq 0$, define $\xi_t = \sum_{\tau = 1}^t m_{\tau} \eta_{\tau}$ and $M_t = I_d + \sum_{\tau = 1}^t m_{\tau} m_{\tau}^T$, where~$I_d$ is the $d-$dimensional identity matrix. Assume~$\eta_t$ is conditionally $R-$sub-Gaussian. Then for any~$\delta' >0$ and $t \geq 0$, with probability at least~$1 - \delta'$,
\[
\|\xi_t\|_{M_t^{-1}} \leq R \sqrt{ d \ln(\frac{t+1}{\delta'})},
\]
where~$\|\xi_t\|_{M_t^{-1}} = \sqrt{\xi_t^{\top} M_t^{-1} \xi_t}$.
\end{lemma}

\begin{lemma} [\cite{auer2002using},~Lemma~11] \label{lem:sum_dev}
  Let~$B(t) = B(t-1) + x_{t-1} x_{t-1}^T$ denote the empirical covariance matrix and let~$\|x_t \|_{B^{-1}(t)} = \sqrt{x_t^T B^{-1}(t) x_t}$ denote the matrix norm associated with~$x_t$. We have 
  \[
  \sum_{t=1}^T \|x_t \|_{B^{-1}(t)} \leq 5 \sqrt{dT \ln(T)}.
  \]
\end{lemma}

\begin{lemma}[\cite{agrawal2013thompson},~Lemma~6]
    For a Gaussian distributed random variable~$Z$ with mean~$m$ and variance~$\sigma^2$, for any~$z\geq 1$,
    \[
    \frac{1}{2\sqrt{\pi}z}e^{\frac{-z^2}{2}} \leq \mathbb{P}(|Z-m|>z\sigma) \leq \frac{1}{\sqrt{\pi}z}e^{\frac{-z^2}{2}}.
    \]
\end{lemma}

\section{Proof of Theorem~1}
Consider the following update equations from Algorithm~\ref{alg:main}
\begin{align*}
    \hat{B}_{a}(t+1)&=\hat{B}_{a}(t)+\hat{x}_t \hat{x}_t^{\top}\1(a=a_t),\\
\hat{f}_{a}(t+1)&=\hat{f}_{a}(t)+\hat{x}_t r_t\1(a=a_t), \\
\hat{\mu}_{a}(t+1)&=\hat{B}_{a}^{-1}(t+1) \hat{f}_{a}(t+1).
\end{align*}
Here $\hat{B}_a(t)$ is the empirical covariance matrix and $\hat{\mu}_a(t)$ is the least squares estimate associated with action~$a$, constructed using the observed contexts and rewards.

\noindent Let the best arm under full observability and partial observability be denoted as
\begin{align*}
    a^*(t) &= \argmax_a x_t^{\top} \mu_a ~\text{and,}\\
    a^{+}(t) &= \argmax_a \hat{x}_t^{\top} \mu_a~~\text{respectively}.
\end{align*} 

\noindent
Let the gap associated with each of the best arms and any generic arm be given as 
\begin{align*}
    \Delta_{a}(t) &= x_t^{\top} \mu_{a^*(t)} - x_t^{\top} \mu_{a}, \\
    \hat{\Delta}_{a}(t) &=  \hat{x}_t^{\top} \mu_{a^{+}(t)} - \hat{x}_t^{\top} \mu_{a}.
\end{align*}


\begin{itemize}
    \item[\textbf{A1}] $\|\mu_a\| \leq M_1,\,\forall a \in \Acal$, for some~$M_1< \infty$. 
    \item[\textbf{A2}] $\|x_t \|, \|\hat{x}_t \| \leq M_2,\,\forall t=1,\cdots,T$, for some $M_2<\infty$.
    \item[\textbf{A3}] There exists a constant $b^{\delta}<\infty$ such that with probability at least $1 - \frac{\delta}{2t^2}$
    \[
    \textstyle\|\hat{B}_a^{-1}(t)\|\leq b^{\delta}/t,\;\forall a\in\Acal,\,t=1,\cdots,T.
    \] 
\end{itemize}
\noindent We point out a fact important for the later analysis. We define $\eta_a \triangleq r_a - x^{\top} \mu_a$. It holds that $\eta_a$ is conditionally~$\sigma-$sub-Gaussian since it is exactly Gaussian with variance $\sigma^2$ from Eq.~\eqref{eq:linear_reward}. 
The instantaneous regret of taking action $a$ in time $t$ is
\[\Delta_a(t)=x_t^{\top}\mu_{a^{\star}(t)}-x_t^{\top}\mu_{a}.\]
The cumulative regret of an algorithm over $T$ intervals is
\[\textstyle\Rcal(T)=\sum_{t=1}^{T}\Delta_t(a_t),\]
where actions $\{a_t\}$ are selected by the algorithm.

\subsection{Instantaneous Regret Decomposition} 
To characterize the regret of Algorithm~\ref{alg:main}, we need a bound on how the gaps~$\Delta_t$ accumulate over time. However, since we do not observe the true contexts but work with the estimated ones, we need to first characterize how the gaps~$\hat{\Delta}_t$ accumulate over time and how these gaps relate to~$\Delta_t$. Let $\varepsilon_t = x_t^{\top} - \hat{x}_t^{\top}$ denote the estimation error. 

\begin{proposition} \label{prop:ireg_dec}
Suppose \textbf{A1} holds. For any~$a$ and $t$, the instantaneous regret under full observation
\[
\Delta_a(t) \leq O(\|\varepsilon_t\|) + \hat{\Delta}_t.
\]
\end{proposition}
\begin{proof}
The instantaneous regret is upper bounded as follows.
\begin{align*}
    \Delta_{a}(t) &= x_t^{\top} \mu_{a^*(t)} - x_t^{\top} \mu_{a}, \\
    &= x_t^{\top} \mu_{a^*(t)} - \hat{x}_t^{\top} \mu_{a^*(t)} + \hat{x}_t^{\top} \mu_{a^*(t)} + \hat{x}_t^{\top} \mu_{a^+(t)} - \hat{x}_t^{\top} \mu_{a^+(t)} \\
    & \hspace{2cm} + \hat{x}_t^{\top} \mu_{a} - \hat{x}_t^{\top} \mu_{a} - x_t^{\top} \mu_{a}  \\
    &\leq \Big\{ x_t^{\top} \mu_{a^*(t)} - \hat{x}_t^{\top} \mu_{a^*(t)} \Big\} + \Big\{\hat{x}_t^{\top} \mu_a - x_t^{\top} \mu_{a}\Big\} + \hat{\Delta}_{a}(t) \\
    &= 2 M_1 \cdot \|\varepsilon_t\| + \hat{\Delta}_{a}(t). 
\end{align*}
\end{proof}

\subsection{Estimated Deviation} 
While the previous section derived the relation between desired and realized regret, it is important to note that only reward model parameters that are estimated using estimated contexts are observable. Here we characterize the deviation of the predicted rewards using estimated and true model parameters. We define 
\[\Phi_a(t) \triangleq x_t^{\top}\mu_a - \hat{x}_t^{\top} \mu_a =  \varepsilon_t^{\top} \mu_a.\]

\begin{proposition}
    Let~$\hat{\mu}_a(t)$ denote the least squares estimate when using the estimated contexts. Under assumptions~\textbf{A1} - \textbf{A3}, we have with probability at least~$1 - \delta/t^2$,
    \[
    |\hat{x}^{\top}(t) \hat{\mu}_a(t) - \hat{x}^{\top}(t) \mu_a| \leq \|\hat{x}_t \|_{\hat{B}_a^{-1}(t)} \Bigg(\sigma\sqrt{\ln(\frac{2t^3}{\delta})} + M_1 + \frac{M_1 M_2\sqrt{b}}{t^{1/2}}\sum_{\tau}\|\varepsilon_{\tau}\| \Bigg).
    \]
\end{proposition}

\begin{proof}
Let $\eta_a = r_a - x^{\top} \mu_a$ be a martingale difference process. We have 
\begin{align*}
    r_a - x^{\top} \mu_a &= r_a - (x^{\top}\mu_a - \hat{x}^{\top} \mu_a) - \hat{x}^{\top} \mu_a \\
    &= r_a - \Phi_a - \hat{x}^{\top} \mu_a.
\end{align*}
Let the stochastic process~$\xi_t$  be constructed as 
\begin{align*}
    \xi_t &= \sum_{\tau} \hat{x}_{\tau} \Big(r_a(\tau) - \hat{x}_{\tau}^{\top} \mu_a - \Phi_a(\tau) \Big) \\
    &= \sum_{\tau} \hat{x}_{\tau} \Big(r_a(\tau) - \hat{x}_{\tau}^{\top} \mu_a \Big) - \sum_{\tau} \hat{x}_{\tau} \Phi_a(\tau)
\end{align*}
We have 
\begin{align*}
    \hat{B}_a^{-1}(t) (\xi_t - \mu_a) &= \hat{\mu}_a(t) - \mu_a - \hat{B}_a^{-1}(t)\Big( \sum_{\tau} \hat{x}_{\tau} \Phi_a(\tau)\Big). \\
    \therefore \hat{\mu}_a(t) - \mu_a &= \hat{B}_a^{-1}(t) (\xi_t - \mu_a) + \hat{B}_a^{-1}(t)\Big( \sum_{\tau} \hat{x}_{\tau} \Phi_a(\tau)\Big).
\end{align*}
Taking inner product with~$\hat{x}(t)$, we have using \textbf{A1},
\begin{align*}
    |\hat{x}^{\top}(t) \hat{\mu}_a(t) - \hat{x}^{\top}(t) \mu_a| &\leq |\hat{x}^T(t) \hat{B}_a^{-1}(t) (\xi_t - \mu_a)| + |\hat{x}^T(t) \hat{B}_a^{-1}(t)\Big( \sum_{\tau} \hat{x}_{\tau} \Phi_a(\tau)\Big)| \\
    &\leq \|\hat{x}_t \|_{\hat{B}_a^{-1}(t)} \cdot \|\xi_t - \mu_a \|_{\hat{B}^{-1}(t)} + \|\hat{x}_t \|_{\hat{B}_a^{-1}(t)} \cdot \| \sum_{\tau} \hat{x}_{\tau} \Phi_a(\tau) \|_{\hat{B}_a^{-1}(t)} \\
    &\leq \|\hat{x}_t \|_{\hat{B}_a^{-1}(t)} \cdot \Bigg\{ \|\xi_t \|_{\hat{B}_a^{-1}(t)} + M_1 \Bigg\} + \|\hat{x}_t \|_{\hat{B}_a^{-1}(t)} \cdot \|\sum_{\tau} \hat{x}_{\tau} \Phi_a(\tau)\|_{\hat{B}_a^{-1}(t)} \\
    & \leq \|\hat{x}_t \|_{\hat{B}_a^{-1}(t)} \cdot \Bigg\{\|\xi_t \|_{\hat{B}_a^{-1}(t)} + M_1 + \|\sum_{\tau} \hat{x}_{\tau} \Phi_a(\tau) \|_{\hat{B}_a^{-1}(t)} \Bigg\}.
\end{align*}

The term $\|\sum_{\tau} \hat{x}_{\tau} \Phi_a(\tau) \|_{\hat{B}_a^{-1}(t)}$ can be bounded by the prediction error under \textbf{A2} and \textbf{A3}
\begin{align*}
\|\sum_{\tau} \hat{x}_{\tau} \Phi_a(\tau) \|_{\hat{B}_a^{-1}(t)}&=\|\hat{B}_a^{-1/2}(t)\sum_{\tau}\hat{x}_{\tau} \Phi_a(\tau)\|\notag\\
&\leq \|\hat{B}_a^{-1/2}(t)\|\|\sum_{\tau}\hat{x}_{\tau} \varepsilon_{\tau}^{\top}\mu_a\|\notag\\
&\leq \|\hat{B}_a^{-1}(t)\|^{1/2}\sum_{\tau}\|\hat{x}_{\tau}\|\| \varepsilon_{\tau}\|\|\mu_a\|\notag\\
&\leq \frac{M_1 M_2\sqrt{b}}{t^{1/2}}\sum_{\tau}\|\varepsilon_{\tau}\|.
\end{align*}

By Lemma~\ref{lem:mnorm_ub} with $\delta' = \frac{\delta}{2t^2}$ and $R=\sigma$, and using \textbf{A1} - \textbf{A3}, we have w.p at least~$1 - \frac{\delta}{t^2}$
\[
\|\xi_t \|_{\hat{B}_a^{-1}(t)} + M_1 + \|\sum_{\tau} \hat{x}_{\tau} \Phi_a(\tau) \|_{\hat{B}_a^{-1}(t)} \leq \sigma \sqrt{d \ln(\frac{2t^3} {\delta})} + M_1 + \frac{M_1 M_2\sqrt{b}}{t^{1/2}}\sum_{\tau}\|\varepsilon_{\tau}\|.
\]
\end{proof}

\subsection{Sampling Deviation}
The parameters estimated using least squares along with the empirical covariance matrix modulate the Gaussian distribution from which the parameters for action selection are sampled from in Algorithm~\ref{alg:main}. Let~$\theta_a(t) = \hat{x}^T_t \tilde{\mu}_a(t)$, where~$\tilde{\mu}_a$ is sampled from a multivariate Gaussian with mean at~$\hat{\mu}_a(t)$. Here we characterize the remaining deviation, i.e, we bound~$|\theta_a(t) - \hat{x}_t^T \hat{\mu}_a(t)|$. 
Suppose the likelihood of reward~$r_a(t)$ at time~$t$ given the estimated context~$\hat{x}_t$ and parameter~$\beta_a$ were given by the pdf of a Gaussian distribution~$\mathcal{N}(\hat{x}_t^T \beta_a, v_t^2)$. Then if the prior for~$\beta$ at time~$t$ is given by~$\mathcal{N}(\hat{\mu}_a(t),v_t^2 \hat{B}^{-1}(t))$, it is well known that the posterior is again Gaussian distributed and is proportional to~$\mathcal{N}(\hat{\mu}_a(t+1), v_{t+1}^2 \hat{B}^{-1}(t+1))$.

\begin{proposition}
With probability at least~$1 - \frac{1}{t^2}$, we have
\[
|\theta_a(t) - \hat{x}_t^T \hat{\mu}_a(t)| \leq  \|\hat{x}_t \|_{\hat{B}^{-1}_a(t)} \cdot \min\Bigg\{ \sqrt{ 4d \ln(t)}, \sqrt{4 \ln (Kt)} \Bigg\}v_t.
\]
\end{proposition}

\begin{proof}
The proof follows using the same arguments as in~\cite[Lemma~1]{agrawal2013thompson}.    
\end{proof}

\subsection{Instantaneous Regret Bound}
Now that we have derived the deviations of the observed from the desired quantities, we proceed to derive an upper bound on the instantaneous regret. Let~$g_t = \min\Bigg\{ \sqrt{ 4d \ln(t)}, \sqrt{4 \ln (Kt)} \Bigg\}v_t + l_t$, where $l_t = \Bigg(\sigma \sqrt{d \ln(\frac{2t^3} {\delta})} + M_1 + \frac{M_1 M_2\sqrt{b}}{t^{1/2}}\sum_{\tau}\|\varepsilon_{\tau}\| \Bigg)$. 

\noindent Let~$E^{\mu}(t)$ denote the event that
\[
\forall~a: |\hat{x}^{\top}(t) \hat{\mu}_a(t) - \hat{x}^{\top}(t) \mu_a| \leq \|\hat{x}_t \|_{\hat{B}_a^{-1}(t)} \cdot l_t,
\]
and let~$E^{\theta}(t)$ denote the event that 
\[
\forall~a:|\theta_a(t) -  \hat{x}^{\top}(t) \hat{\mu}_a(t)| \leq \|\hat{x}_t \|_{\hat{B}^{-1}_a(t)} \cdot \min\Bigg\{ \sqrt{ 4d \ln(t)}, \sqrt{4 \ln (Kt)} \Bigg\}v_t.
\]

\begin{proposition} \label{prop:inst_reg}
    Let~$p = \frac{1}{4e \sqrt{\pi}}$. For any filtration~$\mathcal{F}_{t-1}$ such that~$E^{\mu}$ is true,
    \[
    \mathbb{E}\Big[ \hat{\Delta}_a(t) | \mathcal{F}_{t-1} \Big] \leq \frac{3 g_t}{p} \mathbb{E} \Big[ \|\hat{x}(t) \|_{\hat{B}^{-1}_a(t)} | \mathcal{F}_{t-1} \Big] + \frac{2 g_t}{p t^2}.
    \]
\end{proposition}
\begin{proof}
    The proof follows using the same arguments as in~\cite[Lemma~4]{agrawal2013thompson}.
\end{proof}

\begin{proposition}
    Let~$\texttt{regret}(t) = \hat{\Delta}_a(t) = \hat{x}_t^{\top} \mu_{a^{+}(t)} - \hat{x}_t^{\top} \mu_{a}$ and $\texttt{regret}'(t) = \texttt{regret}(t) \cdot I(E^{\mu}(t))$. The following holds:
    \[
    \mathbb{E}\Big[ \texttt{regret}'(t) | \mathcal{F}_{t-1} \Big] \leq \frac{3 g_t}{p} \mathbb{E} \Big[ \|\hat{x}(t) \|_{\hat{B}^{-1}_a(t)} | \mathcal{F}_{t-1} \Big] + \frac{2 g_t}{p t^2}.
    \]
\end{proposition}

\begin{proof}
    Note that $I(E^{\mu}(t))$ is measurable with respect to the filtration~$\mathcal{F}_{t-1}$, hence it is either~$1$ or $0$. Result follows from Proposition~\ref{prop:inst_reg}.
\end{proof}

\subsection{Total Regret Bound}
\begin{thm}
    Suppose the assumptions \textbf{A1}-\textbf{A3} hold with~$M_1=M_2=1$. With probability~$1-\delta$, the regret of the partially observed Thompson Sampling algorithm is bounded as 
    \[
    \mathcal{R}(T) = O\Bigg(d^{3/2}\sqrt{T \ln(K)} \Big(\ln(T) + \sqrt{\ln(T) \ln(\frac{1}{\delta})}\Big) + \sqrt{d \ln(T)} \cdot \sum_{t=1}^{T} \|\varepsilon_{t}\| \Bigg) 
    \]
\end{thm}

\begin{proof}
Defining
\[
X_t = \texttt{regret}'(t) - \frac{2g_t}{p} \|\hat{x}_a (t)\|_{\hat{B}^{-1}_a(t)} - \frac{2 g_t}{pt^2},
\]
and assuming that~$\hat{\Delta}_a(t) \leq 1$ for all~$a,t$, we have that 
\[
|X_t| \leq 1 + \frac{3g_t}{p} + \frac{2}{pt^2} \leq \frac{6g_t}{p}.
\]
By Azuma-Hoeffding inequality for a the super-martingale~$X_t$, we have with probability~$1 - \frac{\delta}{2}$,
\[
\sum_{t=1}^T \texttt{regret}'(t) \leq \sum_{t=1}^T \Bigg\{\frac{3 g_t}{p}  \|\hat{x}(t) \|_{\hat{B}^{-1}_a(t)}  + \frac{2 g_t}{p t^2} \Bigg\} + \sqrt{2 \Big( \sum_t \frac{36g_t^2}{p^2}\Big)\ln(\frac{2}{\delta})}.
\]
Noting that~$g_t \leq g_T\Big(:= O(\sqrt{d \ln(T/\delta)} \cdot \min\{\sqrt{d}, \sqrt{\log(K)} \} + \frac{1}{\sqrt{T}}\sum_{\tau=1}^T \|\varepsilon_{\tau} \|)\Big)$, we have with probability~$1 - \frac{\delta}{2}$,
\[
\sum_{t=1}^T \texttt{regret}'(t) \leq \frac{3g_T}{p} \sum_{t=1}^T \|\hat{x}(t) \|_{\hat{B}^{-1}_a(t)}  + \frac{2g_T}{p} \sum_t \frac{1}{t^2} + \frac{6g_T}{p} \sqrt{2T \ln(\frac{2}{\delta})}.
\]
Since~$E^{\mu}(t)$ holds for all~$t$ with probability~$1 - \delta/2$, we have $\texttt{regret}'(t) = \texttt{regret}(t)$ with probability at least~$1 - \delta/2$. Therefore, from Proposition~\ref{prop:ireg_dec} and using Lemma~\ref{lem:sum_dev}, we have
\[
\mathcal{R}(T) \leq O\Bigg(d\sqrt{T} \cdot \min\{\sqrt{d}, \sqrt{\log(K)} \} \cdot \Big(\ln(T) + \sqrt{\ln(T) \ln(1/\delta)} \Big) + \sqrt{d \ln(T)} \cdot \sum_{\tau=1}^T \|\varepsilon_{\tau} \| \Bigg).
\]
\end{proof}

\end{document}